\documentclass[nohyperref]{article}

\usepackage{microtype}
\usepackage{graphicx}
\usepackage{subfigure}
\usepackage{booktabs} 

\usepackage{hyperref}

\usepackage[accepted]{icml2022}

\usepackage{amsmath}
\usepackage{amssymb}
\usepackage{mathtools}
\usepackage{amsthm}
\usepackage[capitalize,noabbrev]{cleveref}

\theoremstyle{plain}
\newtheorem{theorem}{Theorem}[section]
\newtheorem{proposition}[theorem]{Proposition}
\newtheorem{lemma}[theorem]{Lemma}
\newtheorem{corollary}[theorem]{Corollary}
\theoremstyle{definition}
\newtheorem{definition}[theorem]{Definition}
\newtheorem{assumption}[theorem]{Assumption}
\theoremstyle{remark}
\newtheorem{remark}[theorem]{Remark}

\usepackage{amssymb}
\usepackage{pifont}
\newcommand{\cmark}{\ding{51}}%
\newcommand{\xmark}{\ding{55}}%

\usepackage[textsize=tiny]{todonotes}

\icmltitlerunning{ProbAbilistic Gradient Estimation for Policy Gradient (PAGE-PG)}

\begin{document}

\twocolumn[
\icmltitle{PAGE-PG: A Simple and Loopless Variance-Reduced Policy Gradient Method with Probabilistic Gradient Estimation}



\icmlsetsymbol{equal}{*}

\begin{icmlauthorlist}
\icmlauthor{Matilde Gargiani}{yyy}
\icmlauthor{Andrea Zanelli}{zzz}
\icmlauthor{Andrea Martinelli}{yyy}
\icmlauthor{Tyler Summers}{xxx}
\icmlauthor{John Lygeros}{yyy}
\end{icmlauthorlist}

\icmlaffiliation{yyy}{Automatic Control Laboratory, ETH Zurich, Switzerland}
\icmlaffiliation{zzz}{Department of Mechanical and Process Engineering, ETH Zurich, Switzerland}
\icmlaffiliation{xxx}{Department of Mechanical Engineering, University of Texas at Dallas}

\icmlcorrespondingauthor{Matilde Gargiani}{gmatilde@ethz.ch}

\icmlkeywords{Machine Learning, ICML}

\vskip 0.3in
]


\printAffiliationsAndNotice{}

\begin{abstract}
Despite their success, policy gradient methods suffer from high variance of the gradient estimate, which can result in unsatisfactory sample complexity.
Recently, numerous variance-reduced extensions of policy gradient methods with provably better sample complexity and competitive numerical performance have been proposed.
After a compact survey on some of the main variance-reduced REINFORCE-type methods, we propose ProbAbilistic Gradient Estimation for Policy Gradient (PAGE-PG), a
novel loopless variance-reduced policy gradient method based on a probabilistic switch between two types of updates. Our method is inspired by the PAGE estimator for supervised learning and leverages importance sampling to obtain an unbiased gradient estimator. We show that PAGE-PG enjoys a $\mathcal{O}\left( \epsilon^{-3} \right)$ average sample complexity to reach an $\epsilon$-stationary solution, which matches the sample complexity of its most competitive counterparts under the same setting. 
A numerical evaluation confirms the competitive performance of our method on classical control tasks.
\end{abstract}

\section{Introduction}
\label{introduction}
Policy gradient methods have proved to be really effective in many challenging deep reinforcement learning (RL) applications~\cite{sutton2018reinforcement}. Their success is also due to their versatility as they are applicable to any differentiable policy parametrization, including complex neural networks, and they admit easy extensions to model-free settings and continuous state and action spaces. This class of methods has a long history in the RL literature that dates back to~\cite{williams1992}, but only very recent work~\cite{agarwal2019} has characterized their theoretical properties, such as convergence to a globally optimal solution and sample and iteration complexity. Since in RL it is generally not possible to compute the exact gradient, but we rely on sample-based approximations, policy gradient methods are negatively affected by the high-variance of the gradient estimate, which slows down convergence and leads to unsatisfactory sample complexity. To reduce the variance of the gradient estimators, actor-critic methods are deployed, where not only the policy, but also  the state-action value function or the advantage function are parameterized~\cite{pmlr-v48-mniha16}. Alternatively, taking inspiration from stochastic optimization, various variance-reduced policy gradient methods have been proposed~\cite{sidford2018, papini2018, xu2019, xu2021, yuan2020,Wang2021}. 

In this work, we focus on variance-reduced extensions of REINFORCE-type methods, such as REINFORCE~\cite{williams1992}, GPOMDP~\cite{baxter2001} and their variants with baseline~\cite{sutton2018reinforcement}. After reviewing the principal variance-reduced extensions of REINFORCE-type methods, we introduce a novel variance-reduced policy gradient method, PAGE-PG, based on the recently proposed PAGE estimator for supervised learning~\cite{li2021}. We prove that PAGE-PG only takes $\mathcal{O}\left( \epsilon^{-3}\right)$ trajectories on average to achieve an $\epsilon$-stationary policy, which translates into a near-optimal solution for gradient dominated objectives. This result matches the bounds on total sample complexity of the most competitive variance-reduced REINFORCE-type methods under the same setting. The key feature of our method consists in replacing the double-loop structure typical of variance-reduced methods with a probabilistic switch between two types of updates. According to recent works in supervised learning~\cite{kovalev2020, li2021}, loopless variance-reduced methods are easier to tune, analyze and generally lead to superior and more robust practical behavior. For policy gradient optimization, similar advantages are discussed in~\cite{yuan2020}, where the authors propose STORM-PG, which, to the best of our knowledge, is the only other loopless variance-reduced policy gradient counterpart to our method. With respect to STORM-PG, our method enjoys a better theoretical rate of convergence. In addition, our experiments show the competitive performance of PAGE-PG on classical control tasks. Finally, we describe the limitations of the considered methods, discuss promising future extensions as well as the importance of incorporating noise annealing and adaptive strategies in the PAGE-PG's update. These might favor exploration in the early stages of training and improve convergence in presence of complex non-concave landscapes~\cite{neelakantan2015, smith2018, zhou2019}.

\textbf{Main contributions.} Our main contributions are summarized below.
\begin{itemize}
    \itemsep-0.2em 
    \item We propose PAGE-PG, a novel loopless variance-reduced extension of REINFORCE-type methods based on a probabilistic update.
    \item We show that PAGE-PG enjoys a fast rate of convergence and achieves an $\epsilon$-stationary policy within $\mathcal{O}\left( \epsilon^{-3}\right)$ trajectories on average. We further show that, with gradient dominated objectives, similar results are valid for near-optimal solutions. 
\end{itemize}

\section{Problem Setting}\label{sec: problem setting}

In this section, we describe the problem setting and briefly discuss the necessary background material on REINFORCE-type policy gradient methods.

\noindent\textbf{Markov Decision Process.} The RL paradigm is based on the interaction between an agent and the environment. In the standard setting, the agent observes  the state of the environment and, based on that observation, plays an action according to a certain policy. As a consequence, the environment transits to a next state and a reward signal is emitted from the environment back to the agent. This process is repeated over a horizon of length $H>0$, with $H<\infty$ in the episodic setting and $H\rightarrow\infty$ in the infinite-horizon setting. From a mathematical viewpoint, Markov Decision Processes (MDPs) are a widely utilized mathematical tool to describe RL tasks. In this work we consider discrete-time episodic MDPs $\mathcal{M}=\left\{ \mathcal{S}, \mathcal{A}, P, r, \gamma , \rho \right\}$, where $\mathcal{S}$ is the state space; $\mathcal{A}$ is the action space; $P$ is a Markovian transition model, where $P(s'\,|\,s,a)$ defines the transition density to state $s'$ when taking action $a$ in state $s$; $r:\mathcal{S}\times\mathcal{A}\rightarrow [-R,R]$ is the reward function, where $R>0$ is a constant; $\gamma\in(0,1)$ is the discount factor; and $\rho$ is the initial state distribution. 
The agent selects the actions according to a stochastic stationary policy $\pi$, which, given a state $s$, defines a density distribution over the action space $\pi(\cdot\,|\,s)$. 
A trajectory $\tau= \left\{ s_h, a_h \right\}_{h=0}^{H-1}$ is a collection of states and actions with $s_0\sim\rho$ and, for any time-step $h\geq 0$, $a_h\sim \pi(\cdot\,|\,s_h)$ and $s_{h+1}\sim P(\cdot\,|\,s_h,a_h)$. We denote the trajectory distribution induced by policy $\pi_{\theta}$ as
$p(\tau\,|\,\theta)$.

The value function $V^{\pi}:\mathcal{S}\rightarrow \mathbb{R}$ associated with a policy $\pi$ and initial state $s$ is defined as

$$ V^{\pi}(s) := \mathbb{E}\left[ \sum_{h=0}^{H-1} \gamma^h r(s_h, a_h)\,|\,\pi, s_0=s\right]\,, $$
where the expectation is taken with respect to the trajectory distribution. With an overloaded notation, we denote with $V^{\pi}(\rho)$ the expected value under the initial state distribution $\rho$, i.e.,
\begin{equation}\label{eq: expected_Vpi}
V^{\pi}(\rho):=\mathbb{E}_{s_0\sim\rho} \left[ V^{\pi}(s_0) \right]\,. 
\end{equation}

The goal of the agent generally is to find the policy $\pi$ that maximizes $V^{\pi}(\rho)$ in~\eqref{eq: expected_Vpi}.

\noindent\textbf{Policy Gradient.}
Given finite state and action spaces, the policy can be exactly coded with $\vert \mathcal{S} \vert \times \vert \mathcal{A} \vert$ parameters in the tabular setting. However, the tabular setting becomes intractable for large state and action spaces. In these scenarios, as well as in infinite countable and continuous spaces, we generally resort to parametric function approximations. In particular, instead of optimizing over the full space of stochastic stationary policies, we restrict our attention to the class of stochastic policies that is described by a finite-dimensional differentiable parametrization $\Pi_{\theta} =\left\{\pi_{\theta} \,|\,\theta\in\mathbb{R}^d \right\}$, such as a deep neural network~\cite{levine2014}. The addressed problem therefore becomes
\begin{equation}\label{eq: main_problem}
\max_{\theta \in \mathbb{R}^d} V^{\pi_{\theta}}(\rho)\,.
\end{equation}

We denote with $V^*$ the optimal value. To simplify the notation, we use $V(\theta)$ to denote $V^{\pi_{\theta}}(\rho)$, $\theta$ to denote $\pi_{\theta}$ and $R(\tau)=\sum_{h=0}^{H-1}\gamma^h r(s_h, a_h)$ to denote the discounted cumulative reward associated with trajectory $\tau$.
Problem~\eqref{eq: main_problem} can be addressed via gradient ascent,  which updates the parameter vector by taking fixed steps of length $\eta>0$ along the direction of the gradient. The iterations are defined as
\begin{equation}\label{eq: gd_iteration}
\theta_{t+1} = \theta_t + \eta \nabla_{\theta} V(\theta_t) \,,
\end{equation}
where the gradient is given by
\begin{equation}\label{eq: grad}
\begin{aligned}
\nabla_{\theta} V(\theta) 
&=\mathbb{E}_{\tau\sim p(\cdot\,|\,\theta)} \left[ \sum_{h=0}^{H-1} \nabla_{\theta}\log \pi_{\theta}(a_h\,|\,s_h) R(\tau)\right]\,.
\end{aligned}
\end{equation}

In the model-free setting, we cannot compute the exact gradient as we do not have access to the MDP dynamics. Instead, given a certain policy $\theta$, we simulate a finite number $N>0$ of trajectories, which are then used to approximate Equation~\eqref{eq: grad} via Monte Carlo
\begin{equation}\label{eq: reinforce_grad}
\hat{\nabla}_{\theta} V^{\text{RF}}(\theta) = \frac{1}{N}\sum_{i=1}^N  \sum_{h=0}^{H-1} \nabla_{\theta}\log \pi_{\theta}(a^{i}_h\,|\,s^{i}_h) R(\tau_i)\,,
\end{equation} 
where each trajectory $\tau_i = \left\{ s_h^{i}, a_h^{i}\right\}_{h=0}^{H-1}$ is generated according to the trajectory distribution $p(\cdot\,|\,\theta)$. The estimator in Equation~\eqref{eq: reinforce_grad} is also known as the REINFORCE estimator.
An alternative is given by the GPOMDP estimator
\begin{equation}\label{eq: gpomdp}
\hat{\nabla}_{\theta} V^{\text{GPOMDP}}(\theta) = \frac{1}{N}\sum_{i=1}^N  \sum_{h=0}^{H-1} \gamma^{h}r(s_h^i,a_h^i) Z_{\theta, h}\,,
\end{equation}
where for compactness $Z_{\theta, h}=\sum_{z=0}^h \nabla_{\theta}\log\pi_{\theta}(a_z^i\,|\,s_z^i)$ . Both REINFORCE and GPOMDP provide unbiased estimates of the gradient, but they are not equivalent in terms of variance. Specifically for GPOMDP, by only considering the reward-to-go instead of the full reward, we are removing potentially  noisy  terms  and  therefore  lowering  the  variance  of  our  estimate~\cite{NIPS2011_85d8ce59}.
In addition, since $\mathbb{E}\left[ \nabla_{\theta} \log \pi_{\theta}(a\,|\,s) b(s)\right] = 0$ with $b(s)$ being a function of the state, e.g. the value function $V^{\pi}(s)$, both the REINFORCE and GPOMDP estimators can be used in combination with a baseline.  

The discussed estimators (with or without baseline) are deployed in place of the exact gradient in Equation~\eqref{eq: gd_iteration}, leading to the REINFORCE and GPOMDP algorithms. These methods are reminiscent of stochastic gradient ascent~\cite{bottou2021} that also relies on sample-based estimators of the true gradient. 

\noindent\textbf{Notation.} With an overloaded notation, we use $g(\tau_i\,|\,\theta)=\sum_{h=0}^{H-1} \nabla_{\theta}\log \pi_{\theta}(a^{i}_h\,|\,s^{i}_h) R(\tau_i)$ for the REINFORCE estimator, and $g(\tau_i\,|\,\theta)=\sum_{h=0}^{H-1} \gamma^{h}r(s_h^i,a_h^i) Z_{\theta, h}$ for the GPOMDP estimator. 

\section{Related Work}\label{sec: related work}
Variance-reduction techniques have been first introduced for training supervised machine learning models, such as logistic regression, support vector machines and neural networks.
Supervised learning is often recast into a finite-sum empirical risk minimization problem that in its simplest takes the form
\begin{equation}\label{eq: finite-sum problem}
    \min_{\theta \in\mathbb{R}^d} f(\theta):= \frac{1}{n}\sum_{i=1}^n \ell (f_i(\theta))\,.
\end{equation}
In the supervised learning scenario, $n$ is the size of the training dataset $\mathcal{D} = \left\{(x_i, y_i)\right\}_{i=1}^n$ and $\ell$ is a loss function that measures the discrepancy between the model prediction $f_i(\theta) = f(x_i; \theta)$ and the true value $y_i$. If $f$ is smooth and satisfies the Polyak-Łojasiewicz (PL) condition, gradient descent with an appropriate constant step-size enjoys a global linear rate of convergence~\cite{karimi2020}. Despite its fast rate, often the full gradient computation makes the iterations of gradient descent too expensive because of the large size of the training dataset. Mini-batch gradient descent replaces the full gradient with an estimate computed over a randomly sampled subset of the available samples. The method requires a decreasing step-size to control the variance and achieve convergence. As a consequence, the lower-iteration cost comes at the price of a slower sublinear convergence rate~\cite{karimi2020}. A better trade-off between computational costs and convergence rate is achieved by variance-reduced gradient methods, such as SVRG~\cite{NIPS2013_ac1dd209}, Katyusha~\cite{allen-zhu2017}, SARAH~\cite{pmlr-v70-nguyen17b}, STORM~\cite{NEURIPS2019_b8002139}, L-SVRG and L-Katyusha~\cite{kovalev2020}, and PAGE~\cite{li2021}. Because of their provably superior theoretical properties and their competitive numerical performance, these methods have attracted great attention from the machine learning community in the past decade. A key structural feature of classical variance-reduced methods, such as SVRG, Katyusha and SARAH, is the double-loop structure. In the outer loop a full pass over the training data is made in order to compute the exact gradient, which is then used in the inner loop together with new stochastic gradient information to construct a variance-reduced estimator of the gradient. However, as underlined in~\cite{zhou2019}, the double-loop structure complicates the analysis and tuning, since the optimal length of the inner loop depends on the value of some structural constants that are generally unknown and often very hard to estimate. This inconvenience has fueled recent efforts from the supervised learning community to develop loopless variance-reduced gradient methods, such as L-SVRG, L-Katyusha and PAGE. In these methods the outer loop is replaced by a probabilistic switch between two types of updates: with probability $p$ a full gradient computation is performed, while with probability $1-p$ the previous gradient is reused with a small adjustment that varies based on the method. In particular, L-SVRG and L-Katyusha, which are designed for smooth and strongly convex objective functions, recover the same fast theoretical rates of their loopy counterparts, SVRG and Katyusha, but they require less tuning and lead to superior and more robust practical behavior~\cite{kovalev2020}. Similar results also hold for PAGE~\cite{li2021}, which is designed for non-convex problems. In particular, in the non-convex finite-sum setting, PAGE achieves the optimal convergence rate and numerical evidence confirms its competitive performance. 

Motivated by the great success of variance-reduction techniques in supervised learning, numerous recent works have explored their deployment in RL, and, in particular, their adaptation for policy gradient~\cite{papini2018, xu2021, yuan2020, Wang2021}. As discussed in Section~\ref{sec: problem setting}, the gradient is generally estimated based on a finite number $N$ of observed trajectories. In online RL, the trajectories are sampled at each policy change. This is particularly costly since it requires one to simulate the system $N$ times for every parameter update. Generally the batch-size is tuned to achieve the best trade-off between the cost of simulating the system and the variance of the estimate. This motivates the use of variance-reduced gradient methods,  
that use past gradients to reduce variance, leading to an improvement in terms of sample complexity. 
The deployment of variance-reduction techniques to solve Problem~\eqref{eq: main_problem} is not straightforward and requires some adaptations to deal with the specific challenges of the RL setting~\cite{papini2018}. Differently from the finite-sum scenario, Problem~\eqref{eq: main_problem} can not be recast as a finite-sum problem, unless both the state and action spaces are finite. This and the fact that the MDP dynamics are unknown prevent from the computation of the full gradient, which is generally replaced by an estimate based on a large batch-size. 

An additional difficulty comes from the fact that the data distribution changes over time, since it depends on the parameter $\theta$, which gets updated during training. This is known as \textit{distribution shift} and requires the deployment of importance weighting in order to reuse past information without adding a bias to the gradient estimate. In particular, suppose we have two policies $\theta_1$ and $\theta_2$, where $\theta_2$ is used for the interaction with the system, while we aim at obtaining an unbiased estimate of the gradient with respect to $\theta_1$. The unbiased off-policy extension of the REINFORCE estimator~\cite{papini2018} is obtained by replacing $g(\tau_i\,|\,\theta)$ in Equation~\eqref{eq: reinforce_grad}  with the following quantity
\begin{equation}
g^{\omega_{\theta_2}}(\tau_i\,|\theta_1) = \omega(\tau_i\,|\,\theta_2,\theta_1)\sum_{h=0}^{H-1} \nabla_{\theta}\log \pi_{\theta_1}(a_h^i\,|\,s_h^i)R(\tau_i)\,,    
\end{equation}
where $\omega(\tau_i\,|\,\theta_2, \theta_1)= \Pi_{j=0}^{H-1} \frac{\pi_{\theta_1}(a_j^i\,|\,s_j^i)}{\pi_{\theta_2}(a_j^i\,|\,s_j^i)}$ is the importance weight for the full trajectory realization $\tau_i$. Similarly, the off-policy extension of the GPOMDP estimator~\cite{papini2018} is obtained by replacing $g(\tau_i\,|\,\theta)$ in Equation~\eqref{eq: gpomdp} with the following quantity
\begin{equation}
g^{\omega_{\theta_2}}(\tau_i\,|\theta_1) =\sum_{h=0}^{H-1} \omega_{0:h}(\tau_i\,|\,\theta_2,\theta_1)\gamma^{h}r(s_h^i,a_h^i) Z_{\theta_1, h}\,, 
\end{equation}
where $\omega_{0:h}(\tau_i\,|\,\theta_2,\theta_1)=\Pi_{j=0}^h \frac{\pi_{\theta_1}(a_j^i\,|\,s_j^i)}{\pi_{\theta_2}(a_j^i\,|\,s_j^i)}$ is the importance weight for the trajectory realization $\tau_i$ truncated at time $h$. 
Clearly, for the importance weights to be well-defined, the policy $\theta_2$ needs to have a non-zero probability of selecting any action in every state. This assumption is implicitly required to hold where needed throughout the paper. It is easy to verify that, for both the off-policy extensions of REINFORCE and GPOMDP, $\mathbb{E}_{\tau\sim p(\cdot\,|\,\theta_1)}\left[g^{\omega_{\theta_2}}(\tau\,|\,\theta_1) \right] = \mathbb{E}_{\tau\sim p(\cdot\,|\,\theta_2)}\left[g(\tau\,|\,\theta_2) \right]$, leading to an unbiased variance-reduced estimate of the gradient.

\subsection{Variance-Reduced REINFORCE-type Methods}
We now briefly review some of the state-of-the-art variance-reduced REINFORCE-type methods to solve Problem~\eqref{eq: main_problem}. We use $g(\tau\,|\,\theta)$ and $g^{\omega_{\theta_2}}(\tau\,|\,\theta_1)$ to refer to both the REINFORCE and GPOMDP estimators. \\
\textbf{Stochastic Varaice-Reduced Policy Gradient (SVRPG)}, first proposed in~\cite{papini2018} and then further analyzed in~\cite{xu2019}, adapts the stochastic variance-reduced gradient method for finite-sum problems~\cite{NIPS2013_ac1dd209} to deal with the RL challenges as discussed above.
The method is characterized by a double loop structure, where the outer iterations are called epochs. At the $s$-th epoch, a snapshot of the current iterate $\theta_0^s$ is taken. Then, $N>>1$ trajectories $\left\{ \tau_i \right\}_{i=1}^N$ are collected based on the current policy and used to compute the gradient estimator $v_0^s=\frac{1}{N}\sum_{i=1}^N g(\tau_i\,|\,\theta^{s}_0)$. For every epoch, $m$ iterations in the inner loop are performed. At the $t$-th iteration of the inner loop with $t=0,\dots,m-1$, the parameter vector is updated by
\begin{equation}\label{eq: svrpg update}
    \theta_{t+1}^s = \theta_{t}^s + \eta v_{t}^s\,,
\end{equation}
where $\eta>0$. Then  $B<<N$ trajectories $\left\{ \tau_j \right\}_{j=1}^B$ are collected according to the current policy $\theta^{s}_{t+1}$ and an estimate of the gradient at $\theta^s_{t+1}$ is produced 
\begin{equation*}
    v_{t+1}^{s} = \frac{1}{B}\sum_{j=1}^B g(\tau_j\,|\,\theta_{t+1}^{s}) + v_0^s - \frac{1}{B}\sum_{j=1}^B g^{\omega_{\theta_{t+1}^s}}(\tau_j\,|\,\theta^{s}_0)\,.
\end{equation*}
After $m$ iterations in the inner loop, the snapshot is refreshed by setting $\theta_{0}^{s+1} = \theta_{m}^s $, and the process is repeated for a fixed number of iterations.  See Algorithm~\ref{alg:svrpg} in Section~\ref{sec: appendix A} of the Appendix.\\
\noindent\textbf{Stochastic Recursive Variance-Reduced Policy Gradient (SRVRPG)}~\cite{xu2021} is inspired from the SARAH method for supervised learning~\cite{pmlr-v70-nguyen17b}. Differently from SVRPG, SRVRPG incorporates in the update the concept of momentum, which helps convergence by dampening the oscillations typical of first-order methods. In particular, the estimate produced in the inner iterations for all $t=0,\dots,m-1$ is
\begin{equation}\label{eq: srvrpg}
\begin{aligned}
    v_{t+1}^s = \frac{1}{B}&\sum_{i=1}^B g(\tau_j\,|\,\theta_{t+1}^{s}) + v_{t}^s -\frac{1}{B}\sum_{i=1}^B g^{\omega_{\theta_{t+1}^s}}(\tau_j\,|\,\theta^{s}_{t})\,,
\end{aligned}
\end{equation}
where $\left\{ \tau_i\right\}_{i=1}^B$ are generated according to policy $\theta^s_{t+1}$ and $v_0^s$ is the large batch-size estimate computed at the $s$-th epoch.
See Algorithm~\ref{alg:srvrpg} in Section~\ref{sec: appendix A} of the Appendix.\\
\noindent\textbf{Stochastic Recursive Momentum Policy Gradient (STORM-PG)}~\cite{yuan2020} blends the key components of STORM~\cite{NEURIPS2019_b8002139}, a state-of-the-art variance-reduced gradient estimator for finite-sum problems, with policy gradient algorithms.   
A major drawback of SVRPG and SRVRPG is the restarting mechanism, namely, the alternation between large and small batches of sampled trajectories which ensures control of the variance. As discussed for the finite-sum scenario, the double-loop structure complicates the theoretical analysis and the tuning procedure. STORM-PG circumvents the issue by deploying an exponential moving averaging mechanism that exponentially discounts the accumulated variance. The method only requires one to collect a large batch of trajectories at the first iteration and then relies on small batch updates. Specifically, STORM-PG starts by collecting $N>>1$ trajectory samples $\left\{ \tau_i\right\}_{i=1}^N$ according to an initial policy $\theta_0$. Those samples are deployed to calculate an initial gradient estimate $v_0= \frac{1}{N} \sum_{i=1}^N g(\tau_i\,|\,\theta_0)$, which is used in place of the gradient to update the parameter vector as in Equation~\eqref{eq: gd_iteration}.
Then $T$ iterations are performed where at the $t$-th iteration the parameter vector is updated as described in Equation~\eqref{eq: gd_iteration}, but replacing the gradient with the following estimate
\begin{equation}\label{eq: grad_estimate_stormpg}
\begin{aligned}
 v_{t} =  \frac{1}{B}\sum_{i=1}^B g(\tau_i\,|\,\theta_{t})& + (1-\alpha)\Bigg[ v_{t-1} \\
 &\left.- \frac{1}{B}\sum_{i=1}^B g^{\omega_{\theta_t}}(\tau_i\,|\,\theta_{t-1})\right]\,,
\end{aligned}
\end{equation}
where $\alpha\in(0,1)$ and $\left\{\tau_i\right\}_{i=1}^B$ are generated with policy $\theta_t$.
Notice that if $\alpha=1$, we recover the REINFORCE method, while if $\alpha=0$ we recover the SRVRPG update. See Algorithm~\ref{alg:storm-pg} in Section~\ref{sec: appendix A} of the Appendix. 
\medskip
\section{PAGE-PG}
PAGE~\cite{li2021} is a novel variance-reduced stochastic gradient estimator for Problem~\eqref{eq: finite-sum problem}, where $f$ is differentiable but possibly non-convex. Let $\mathcal{B}_t$ be a set of randomly selected indices without replacement from $\left\{ 1,\dots, n\right\}$ and $\vert \mathcal{B}_t \vert = B<<n$, where the subscript $t$ refers to the iteration.
The PAGE estimator is based on a small adjustment to the mini-batch gradient estimator. Specifically, it is initialized to the full gradient $g_0 = \nabla_{\theta} f(\theta_0)$ at $\theta_0$. For the subsequent iterations, the PAGE estimator is defined as follows
\begin{equation}\label{eq: page estimate}
\begin{aligned}
    &g_{t} = \\
    &\begin{cases}
    \frac{1}{n}\sum\limits_{i=1}^n \nabla f_i(\theta_{t}) & \text{prob.}\,p_t\\
    \frac{1}{B}\!\!\sum\limits_{i\in\mathcal{B}_t}\!\!\nabla f_i(\theta_{t}) + g_{t-1}\! - \frac{1}{B}\!\!\sum\limits_{i\in\mathcal{B}_t}\!\!\nabla f_i(\theta_{t-1})&\text{prob.}\,1-p_t. 
    \end{cases}
\end{aligned}
\end{equation}
This unbiased gradient estimator is used to update the parameter vector in a gradient-descent fashion 
\begin{equation*}
    \theta_{t+1} = \theta_{t} - \eta g_t\,,
\end{equation*}
where $\eta>0$ is a fixed step-size.
Therefore, PAGE is based on switching with probability $p_t$ between gradient descent and a mini-batch version of SARAH~\cite{pmlr-v70-nguyen17b}. 

As suggested by its name, PAGE-PG is designed by blending the key ideas of PAGE with policy gradient methods. As discussed in Section~\ref{sec: related work}, we can not simply use the PAGE estimator for policy gradient in its original formulation but some adjustments are required. 
In particular, we substitute the exact gradient computations with a stochastic estimate based on a large batch-size $N>>B$. To deal with the distribution shift, we deploy importance weighting in a similar fashion as the variance-reduced policy gradient methods discussed in Section~\ref{sec: related work}. PAGE-PG works by initially sampling $N$ trajectories with the initial policy $\theta_0$ and using those samples to build a solid gradient estimate $v_0 = \frac{1}{N}\sum_{i=1}^N g(\tau_i\,|\,\theta_0)$. For any $t>0$, PAGE-PG deploys the following estimate
\begin{equation}\label{eq: pagepg estimate}
\begin{aligned}
    &v_t =\\ 
    &\begin{cases}
    \frac{1}{N}\!\!\sum\limits_{i=1}^N\! g(\tau_i\,|\,\theta_{t}) & \text{prob.}\,p_t\\
    \frac{1}{B}\!\!\sum\limits_{i=1}^B\! g(\tau_i\,|\,\theta_t) \!+\! v_{t-1} \!\!-\!\! \frac{1}{B}\!\!\sum\limits_{i=1}^B\! g^{\omega_{\theta_t}}(\tau_i\,|\,\theta_{t-1})& \text{prob.}\,1-p_t,
    \end{cases}
\end{aligned}
\end{equation}
where $\tau_i$ is drawn according to policy $\theta_t$ for any $i$. The parameter vector is updated according to Equation~\eqref{eq: gd_iteration}, where $v_t$ is deployed in place of the gradient. See Algorithm~\ref{alg:page-pg} in Section~\ref{sec: appendix A} of the Appendix for a pseudo-code description. As it appears in Equation~\eqref{eq: pagepg estimate}, the double loop-structure that characterizes SVRPG and SRVRPG is replaced by a probabilistic switching between two estimators. 

\subsection{Theoretical Analysis}\label{sec: theoretical analysis}
For the convergence analysis, we focus on the GPOMDP estimator, since it is generally preferred over the REINFORCE one because of its better performance. Therefore in this section we use $g(\tau_i\,|\,\theta) = \sum_{h=0}^{H-1} \gamma^h r(s_h^i, a_h^i)Z_{\theta, h}$ and $g^{\omega_{\theta_2}}(\tau_i\,|\theta_1) =\sum_{h=0}^{H-1} \omega_{0:h}(\tau_i\,|\,\theta_2,\theta_1)\gamma^{h}r(s_h^i,a_h^i) Z_{\theta_1, h}$. We refer to Section~\ref{proof main theory} in the Appendix for the proofs and to Section~\ref{technical lemmas} in the Appendix for the technical lemmas.
After discussing the fundamental assumptions, we focus on studying the sample complexity of PAGE-PG to reach an $\epsilon$-stationary solution. We further show that, when the objective is gradient-dominated, since $\epsilon$-stationarity translates into near-optimality, the derived results are also valid for near-optimal solutions.

The theoretical analysis of variance-reduced policy gradient methods generally focuses on deriving, under certain assumptions, an upper bound on the number of sampled trajectories that are needed to achieve an $\epsilon$-stationary solution.
\begin{definition}[$\epsilon$-stationary solution]\label{definition}
Let $\epsilon>0$. $\theta \in \mathbb{R}^d$ is an $\epsilon$-stationary solution if and only if 
$\Vert \nabla_{\theta} V(\theta) \Vert \leq \epsilon\,. $
\end{definition}
Based on Definition~\ref{definition}, a stochastic policy gradient based algorithm reaches an $\epsilon$-stationary solution if and only if
$\mathbb{E}\left[ \Vert \nabla_{\theta}V({\theta_{\text{out}}}) \Vert^2 \right]\leq \epsilon^2$,
where ${\theta_{\text{out}}}$ is the output of the algorithm after $T$ iterations and the expected value is taken with respect to all the sources of randomness involved in the process. 
Our analysis is based on the following assumptions. 
\begin{assumption}[Bounded log-policy gradient norm]\label{assumption 1}
For any $a\in\mathcal{A}$ and $s\in\mathcal{S}$ there exists a constant $G>0$ such that
$\Vert \nabla_{\theta}\log \pi_{\theta}(a\,|\,s) \Vert \leq G$ for all $\theta\in\mathbb{R}^d$.
\end{assumption}

\begin{assumption}[Smoothness]\label{assumption 2}
$\pi_{\theta}$ is twice differentiable and for any $a\in\mathcal{A}$ and $s\in\mathcal{S}$ there exists a constant $M>0$ such that
$\Vert \nabla^2_{\theta}\log\pi_{\theta}(a\,|\,s) \Vert \leq M$ for all $\theta\in\mathbb{R}^d$.
\end{assumption}

\begin{assumption}[Finite variance]\label{assumption 3}
There exists a constant $\sigma>0$ such that
$ \text{Var}(g(\tau\,|\,\theta))\leq \sigma^2$ for all $\theta\in\mathbb{R}^d$.
\end{assumption}

\begin{assumption}[Finite importance weight variance]\label{ass:finite_IS_var}
For any policy pair $\theta_a, \theta_b\in\mathbb{R}^d$ and with $\tau\sim p(\cdot\,|\,\theta_b)$, the importance weight $\omega(\tau\,|\,\theta_b,\theta_a) = \frac{p(\tau\,|\,\theta_a)}{p(\tau\,|\,\theta_b)}$ is well-defined.
In addition, there exists a constant $W>0$ such that
$\text{Var}(\omega(\tau\,|\theta_b, \theta_a\,)) \leq W$.
\end{assumption}
The same set of assumptions is considered in~\cite{papini2018, xu2019, xu2021, yuan2020}. By analyzing PAGE-PG in the same setting as its counterparts, we are able to compare them from a theoretical viewpoint. 

\begin{remark}\label{remark on assumptions}
While the non-convex optimization community agrees on the rationality of Assumptions~\ref{assumption 1}-\ref{assumption 3}, in~\cite{Wang2021} the authors argue that Assumption~\ref{ass:finite_IS_var} on the boundedness of the importance weight variance is uncheckable and very stringent. In a more limited setting (finite MDPs only) than the one considered in this work, they are able to remove such assumption via the introduction of a gradient-truncation mechanism that provably controls the variance of the importance weights in off-policy
sampling. This approach is for now outside the scope of this work, but can be addressed in future work by adopting a trust region policy optimisation perspective. In practice, to ensure that Assumption~\ref{ass:finite_IS_var} is met, one can resort to small step-sizes so that $p(\tau\,|\,\theta_b)\approx p(\tau\,|\,\theta_a)$ and the weight is bounded. This, however, comes at the cost of slower convergence, as also confirmed by our numerical experiments in Section~\ref{sec: numerical evaluation}.
\end{remark}

For completeness, we report the following fundamental proposition from~\cite{xu2021}, which is used consistently in our proofs.
\begin{proposition}\label{proposition}
Let $\tau_i$ be a realization of $\tau\sim p(\cdot\,|\,\theta_1)$. Under Assumptions~\ref{assumption 1}-\ref{assumption 2}:
\begin{enumerate}
    \item $\Vert g(\tau_i\,|\,\theta_1) - g(\tau_i\,|\,\theta_2) \Vert\leq L \Vert \theta_1 - \theta_2 \Vert$ for all $\theta_1, \theta_2\in\mathbb{R}^d$, where $L \vcentcolon = MR/(1-\gamma)^2 + 2G^2 R/(1-\gamma)^3$, 
    \item $V(\theta)$ is $L$-smooth and twice differentiable, i.e. $\Vert \nabla_{\theta}^2 V(\theta) \Vert\leq L$.
    \item $\Vert g(\tau_i\,|\,\theta) \Vert \leq C_g$ for all $\theta\in\mathbb{R}^d$ and $C_g \vcentcolon = GR/(1-\gamma)^2$.
\end{enumerate}
\end{proposition}

\begin{theorem}
\label{theorem 1}
Suppose that Assumptions~\ref{assumption 1}-\ref{ass:finite_IS_var} hold and select $\eta>0$, $p\in(0,1]$ and $B\in\mathbb{N}$ such that $\eta^2\leq \min\left\{p/(1-p)\cdot B/2C, 1/4L^2\right\}$. The average expected squared gradient norm after $T$ iterations of PAGE-PG satisfies
\begin{equation*}
    \frac{1}{T}\sum_{t=0}^{T-1} \mathbb{E}\left[ \Vert \nabla_{\theta}V(\theta_t) \Vert^2 \right] \leq \frac{2\left(V^* - V(\theta_0)\right)}{\eta T} + \frac{\sigma^2}{N} + \frac{\sigma^2}{pN T}\,.
\end{equation*}

\end{theorem}
See Lemma~\ref{lemma 0} in Section~\ref{technical lemmas} of the Appendix for the definition of $C$.
Theorem~\ref{theorem 1} states that under a proper choice of step-size, batch-size and probability of switching,
the average expected squared gradient norm of the performance function after $T$ iterations of PAGE-PG is in the
order of
$\mathcal{O}\left(\frac{1}{T} + \frac{1}{NT} + \frac{1}{N} \right)$.
The first term $\mathcal{O}\left( \frac{1}{T} \right)$ characterizes the convergence of PAGE-PG, while the second and third terms come from the variance of the stochastic gradient estimate computed at the iterations with large batches. Our convergence rate improves over the rate $\mathcal{O}\left(\frac{1}{T} + \frac{1}{B} + \frac{1}{N} \right)$ of SVRPG~\cite{papini2018} and over the rate  $\mathcal{O}\left(\frac{1}{T} + \frac{1}{B} + \frac{1}{TN} \right)$ of STORM-PG~\cite{yuan2020}, by avoiding the dependency on the small batch-size $B$. Compared to the rate of SRVR-PG $\mathcal{O}\left( \frac{1}{T} + \frac{1}{N} \right)$, our analysis leads to an extra $\mathcal{O}\left(\frac{1}{TN} \right)$ term which arises from the variance of the first gradient estimate. By selecting $p= \frac{B}{N}$ and $B = \mathcal{O}\left(1 \right)$, we recover the rate $\mathcal{O}\left( \frac{1}{T} + \frac{1}{N}\right)$.

\begin{corollary}\label{corollary 1}
Under the conditions of Theorem~\ref{theorem 1}, set $\eta = \sqrt{B}/\sqrt{2CN}$, $p=1/N$, $N = \mathcal{O}\left( \epsilon^{-2}\right)$ and $B = \mathcal{O}\left( 1\right)$. Then $\frac{1}{T}\sum_{t=0}^{T-1}\mathbb{E}\left[\Big\Vert  \nabla_{\theta}V(\theta_t) \Big\Vert ^2\right]\leq \epsilon^2$ within $\mathcal{O}\left( \epsilon^{-3}\right)$ trajectories on average with $\epsilon\rightarrow 0$.
\end{corollary}

Under the considered assumptions, REINFORCE-type methods, including REINFORCE, GPOMDP as well as their variants with baselines, need $\mathcal{O}(\epsilon^{-4})$ samples to achieve an $\epsilon$-stationary solution. By incorporating stochastic variance reduction techniques the complexity can be reduced. In particular, SVRPG achieves an $\mathcal{O}(\epsilon^{-10/3})$ sample complexity~\cite{xu2019} while the more sophisticated SRVRPG~\cite{xu2021} and STORM-PG~\cite{yuan2020} achieve an $\mathcal{O}(\epsilon^{-3})$ sample complexity. According to Corollary~\ref{corollary 1}, PAGE-PG needs on average $\mathcal{O}\left( \epsilon^{-3}\right)$ trajectories to achieve an $\epsilon$-stationary solution, which makes it competitive from a theoretical viewpoint with its state-of-the-art counterparts. The discussed results on sample complexity are summarized in Table~\ref{table: sample-complexity}.
\begin{table}[t]
\caption{Sample complexities of comparable algorithms for finding an $\epsilon$-stationary solution.}
\label{table: sample-complexity}
\vskip 0.05in
\begin{center}
\begin{small}
\begin{sc}
\begin{tabular}{lcccr}
\toprule
\textbf{Method} & \textbf{Sample-Complexity} & \textbf{No-Restart} \\
\midrule
REINFORCE    & $\mathcal{O}\left( \epsilon^{-4} \right)$ & \xmark\\
GPOMDP    & $\mathcal{O}\left( \epsilon^{-4} \right)$ & \xmark\\
SVRPG    & $\mathcal{O}\left( \epsilon^{-10/3} \right)$ & \xmark\\
SRVRPG & $\mathcal{O}\left( \epsilon^{-3} \right)$ & \xmark\\
STORM-PG    & $\mathcal{O}\left( \epsilon^{-3} \right)$ & \cmark \\
PAGE-PG   & $\mathcal{O}\left( \epsilon^{-3} \right)$ & \cmark \\
\bottomrule
\end{tabular}
\end{sc}
\end{small}
\end{center}
\vskip -0.1in
\end{table}

Recent works~\cite{agarwal2019, pmlr-v130-bhandari21a} have shown that, despite its non-concavity, with certain policy parametrizations, such as direct and softmax, the objective of Problem~\eqref{eq: main_problem} is gradient dominated. We complete our theoretical analysis by extending the results of Theorem~\ref{theorem 1} to gradient-dominated objectives. 
\begin{assumption}[Gradient dominancy]\label{ass:gradient dominancy}
There exists a constant $\lambda>0$ such that
$V^* - V(\theta) \leq \lambda \big\Vert \nabla_{\theta}V(\theta) \big\Vert^2$ for all $\theta\in\mathbb{R}^d\,.$
\end{assumption}

\begin{corollary}\label{corollary 2}
Consider the same setting as in Theorem~\ref{theorem 1} where also Assumption~\ref{ass:gradient dominancy} holds. Then,
\begin{equation*}
    V^* - \max_{t\leq T}\mathbb{E}\left[ V(\theta_t)  \right]\leq \frac{2\lambda\left(V^* - V(\theta_0) \right)}{\eta T} + \frac{\sigma^2\lambda}{N} + \frac{\sigma^2\lambda}{pNT}\,. 
\end{equation*}
\end{corollary}

The gradient dominancy condition implies that any stationary policy is also globally optimal. Consequently, as formalized in Corollary~\ref{corollary 2}, the results from Theorem~\ref{theorem 1} are valid for near-optimal policies, with the only difference that in this case the upper bound on the suboptimality is also proportional to the gradient dominancy constant.  
\section{Numerical Evaluation}\label{sec: numerical evaluation}
In this section we numerically evaluate the performance of the discussed variance-reduced policy gradient methods on two state-of-the-art model-free reinforcement learning tasks from OpenAI Gym~\cite{brockman2016openai}.
In order to conduct the numerical evaluation of the discussed methods, we implemented them, along with GPOMDP, in a Pytorch-based toolbox. In addition, the toolbox interfaces OpenAI Gym~\cite{brockman2016openai} allowing the user to easily train RL agents on different environments with the discussed methods. Finally, Pytorch~\cite{NEURIPS2019_9015} provides the possibility of speeding up the computation via the deployment of graphical processing units (GPUs). The toolbox is publicly available at \url{https://gitlab.ethz.ch/gmatilde/vr_reinforce}. 

\subsection{Benchmarks}
For the empirical evaluation of the discussed methods we consider the Acrobot and the Cartpole environments from OpenAI Gym.
\\
\noindent\textbf{Acrobot.}
The Acrobot system comprises two joints and two links, where the joint between the two links is actuated. Initially, the links are hanging downwards, and the goal is to swing the end of the lower link up to a given height. A reward of $-1$ is emitted every time the goal is not achieved. As soon as the target height is reached or $500$ time-steps are elapsed, the episode ends. The state space is continuous with dimension $6$. The action space is discrete and $3$ possible actions can be selected: apply a positive torque, apply a negative torque, do nothing. To model the policy, we use a neural softmax parametrization. In particular, we deploy a neural network with two hidden layers, width 32 for both layers and $\text{Tanh}$ as activation function.

\noindent\textbf{Cartpole.}
The Cartpole system is a classical control environment that comprises a pole attached by an un-actuated joint to a cart that moves along a frictionless track.
The pendulum starts upright, and the goal is to prevent it from falling over. A reward of $+1$ is provided for every time-step that the pole remains within 15 degrees from the upright position. The episode ends when the pole is more than 15 degrees from vertical, or the cart moves more than 2.4 units from its initial position.
The state space is continuous with dimension $4$. The action space is discrete with 2 available actions: apply a force of $+1$ or $-1$ to the cart. As for the Acrobot, to model the policy we use a neural softmax parametrization. In particular, we deploy a neural network with two hidden layers, width 32 for both layers and $\text{Tanh}$ as activation function.
The maximum episode length is set to 200.

We set $N=100$, $B=5$ and $m=10$ and $\gamma=0.9999$, while the step-size and the other hyperparameters are tuned for each individual algorithm using grid search. See Section~\ref{details on benchmarks} in the Appendix for more details on the choice of the hyperparameters. For each algorithm, we run the
experiment 5 times with random initialization of the environments. The curves (solid-lines) are obtained by taking the mean over the independent runs and the shaded areas represent the $\pm \sigma$ standard deviations. 

The experiments in Figures~\ref{fig: acrobot} and~\ref{fig: cartpole} show that, given enough
episodes, all of the algorithms are able to solve the tasks, achieving near-optimal returns. For the Acrobot environment in Figure~\ref{fig: acrobot}, the SRVRPG and GPOMDP algorithms take the biggest number of episodes to find an optimal policy, while STORM-PG and SVRPG are the fastest in terms of number of episodes. This might be due to the step-size, which, for certain methods, needs to be set to particularly small values to enforce finite importance weight variance and ensure convergence. For the Cartpole environment in Figure~\ref{fig: cartpole}, as
expected, the GPOMDP algorithm takes the longest to find the optimal policy, followed in order by SVRPG, SRVRPG, STORM-PG and PAGE-PG. Notice that these empirical observations corroborate the theoretical findings on the sample complexity. Finally, our benchmarks demonstrate the competitive performance of PAGE-PG with respect to its counterparts.

\begin{figure}[ht]
\vskip 0.01in
\begin{center}
\centerline{\includegraphics[width=\columnwidth]{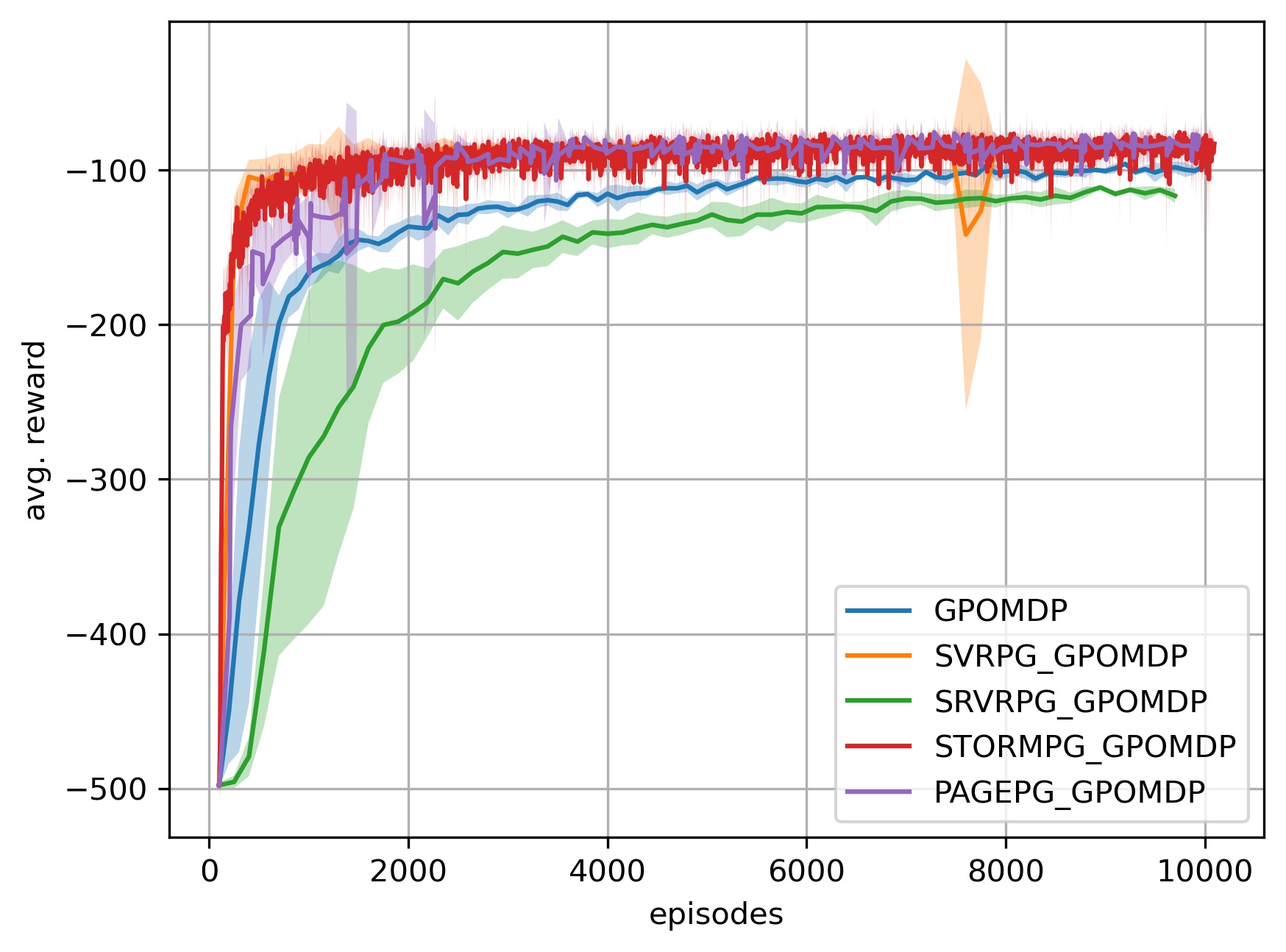}}
\caption{Average reward versus number of episodes for GPOMDP (blue), SVRPG (orange), SRVRPG (green), STORM-PG (red) and PAGE-PG (light purple) on the Acrobot environment. The solid line represents the mean and the shaded areas are calculated as the $\pm\sigma$ of the outcomes over 5 independent runs.}
\label{fig: acrobot}
\end{center}
\vskip -0.2in
\end{figure}

\begin{figure}[ht]
\vskip 0.01in
\begin{center}
\centerline{\includegraphics[width=\columnwidth]{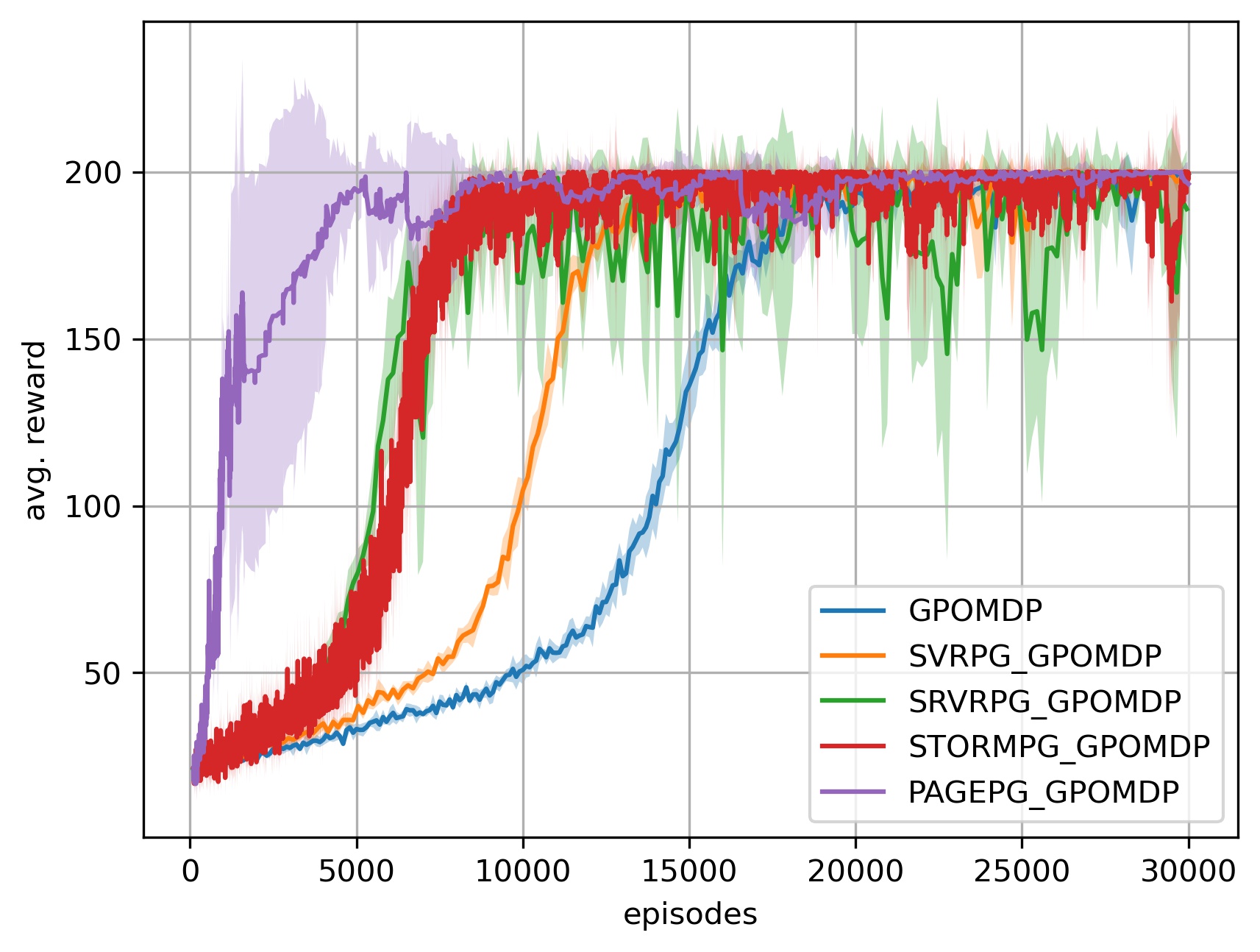}}
\caption{Average reward versus number of episodes for GPOMDP (blue), SVRPG (orange), SRVRPG (green), STORM-PG (red) and PAGE-PG (light purple) on the Cartpole environment. The solid line represents the mean and the shaded areas are calculated as the $\pm\sigma$ of the outcomes over 5 independent runs.}
\label{fig: cartpole}
\end{center}
\vskip -0.2in
\end{figure}

\section{Conclusions, Limitations \& Future Works}
After a brief survey on the main variance reduced policy gradient methods based on REINFORCE-type algorithms, we formulate a novel variance-reduced extension, PAGE-PG, inspired from the PAGE gradient estimator for optimization of non-convex finite-sum problems. To the best of the authors' knowledge, our method is the first variance-reduced policy gradient method that replaces the outer loop with a probabilistic switch. This key feature of PAGE-PG facilitates the theoretical analysis while preserving a fast theoretical rate and a low sample complexity. In addition, our numerical evaluation shows that PAGE-PG has a competitive performance with respect to its counterparts. 

Our benchmarks and theoretical results on the sample complexity confirm that variance-reduced techniques successfully manage to reduce the sample complexity of REINFORCE-type algorithms, speeding up the convergence in terms of number of sampled trajectories. At the same time, it is possible to identify the following limitations: 
\\
\noindent\textbf{Unrealistic and uncheckable assumption on importance weight variance.} As underlined in Remark~\ref{remark on assumptions}, all the discussed variance-reduced policy gradient methods heavily rely on the stringent and uncheckable assumption that the importance
weights have bounded variance for every iteration of the algorithms (Assumption~\ref{ass:finite_IS_var}). To enforce indirectly this assumption, very small values of the step-size are needed, resulting in a dramatic slow-down of the convergence rate. A more efficient alternative could be the deployment of a gradient-truncation strategy, as proposed in~\cite{Wang2021} for the case of finite MDPs. 
This modification, which corresponds to the solution of a trust-region subproblem, is simple and efficient since it does not involve significant extra computational costs but, at the same time, requires to migrate from vanilla REINFORCE-type methods to trust-region based algorithms such as TRPO~\cite{pmlr-v37-schulman15} and PPO~\cite{schulman2017} for comparisons. 

\noindent\textbf{Extreme sensitivity to hyperparameters.} Our benchmarks suggest an extreme sensitivity to the hyperparameters, especially the choice of the step-size. Time-consuming and resource-expensive tuning procedures are required to select a proper configuration of hyperparameters. To alleviate this issue, the update direction should be computed also taking into account second-order information. Second-order methods are notably more robust against the step-size selection than first-order methods, since their update includes information on the local curvature~\cite{agarwal2019, gargiani2020}. \\
\noindent\textbf{Noise annealing strategies.}
Empirical evidence suggests that, in the presence of complex non-concave landscapes, exploration in the form of noise injection is of critical importance in the early stages of training to prevent convergence to spurious local maximizers. Entropy regularization is often used to improve exploration, since it indirectly injects noise in the training process by favoring the selection of more stochastic policies~\cite{pmlr-v97-ahmed19a}. Unfortunately, by adding a regularizer to Problem~\eqref{eq: main_problem} we are effectively changing the optimal policy. An alternative approach would be increasing the batch-size during training. In this perspective, a promising heuristic to further improve the convergence of PAGE-PG could consist in gradually increasing the probability of switching $p_t$. Finally, as also pointed out in~\cite{papini2018}, since the variance of the updates depends on the snapshot policy as well as on the sampled trajectories, it is realistic to imagine that predefined schemes for the probability of switching are not going to perform as well as adaptive ones, which adjust the value of $p_t$ based on some measure of the variance.

We leave for future development the aforementioned extensions, which we believe would counteract the current limitations of the analyzed methods.

\section*{Acknowledgements}
This work has been supported by the European Research Council
(ERC) under the H2020 Advanced Grant no. 787845 (OCAL). The authors thank Niao He for the insightful discussions. 
\bibliography{bibliography}
\bibliographystyle{icml2022}

\newpage
\appendix
\onecolumn
\section{Algorithmic Description of Variance-Reduced REINFORCE-type Methods}\label{sec: appendix A}

\begin{algorithm}[H]
   \caption{SVRPG}
   \label{alg:svrpg}
\begin{algorithmic}
   \STATE {\bfseries Input:} initial parameter $\theta_0$, large batch-size $N$, small batch-size $B$, step-size $\eta>0$, inner-loop length $m\in\mathbb{N}$, number of epochs $S\in\mathbb{N}$
   \STATE initialize $\theta_{m}^{\,0} = \theta_0$
   \FOR{$s=1$ {\bfseries to} $S$}
   \STATE set $\theta^{s}_{0}=\theta^{s-1}_{m}$
   \STATE collect $N$ trajectories with policy $\theta^{s}_{0}$
   \STATE $v^{s}_0 = \frac{1}{N}\sum_{i=1}^N g(\tau_i\,|\,\theta^s_0)$
   \FOR{$t=0$ {\bfseries to} $m-1$}
   \STATE $\theta^{s}_{t+1} = \theta^{s}_{t} + \eta v^{s}_t$
   \STATE sample $B$ trajectories with policy $\theta^s_{t+1}$
   \STATE $v^s_{t+1} = \frac{1}{B}\sum_{i=1}^B g(\tau_i\,|\,\theta^s_{t+1}) + v^s_0 - \frac{1}{B}\sum_{i=1}^B g^{\omega_{\theta^s_{t+1}}}(\tau_i\,|\,\theta^s_0)$
   \ENDFOR
   \ENDFOR
   \STATE {\bfseries Output:} $\theta_{\text{out}}$ chosen uniformly at random from $\left\{ \theta^s_t \right\}_{t=1,\,s=1}^{m,\,S}$
\end{algorithmic}
\end{algorithm}

\begin{algorithm}[H]
   \caption{SRVRPG}
   \label{alg:srvrpg}
\begin{algorithmic}
   \STATE {\bfseries Input:} initial parameter $\theta_0$, large batch-size $N$, small batch-size $B$, step-size $\eta>0$, inner-loop length $m\in\mathbb{N}$, number of epochs $S\in\mathbb{N}$
   \STATE initialize $\theta_{m}^{\,0} = \theta_0$
   \FOR{$s=1$ {\bfseries to} $S$}
   \STATE set $\theta^{s}_{0}=\theta^{s-1}_{m}$
   \STATE collect $N$ trajectories with policy $\theta^{s}_{0}$
   \STATE $v^{s}_0 = \frac{1}{N}\sum_{i=1}^N g(\tau_i\,|\,\theta^s_0)$
   \FOR{$t=0$ {\bfseries to} $m-1$}
   \STATE $\theta^{s}_{t+1} = \theta^{s}_{t} + \eta v^{s}_t$
   \STATE sample $B$ trajectories with policy $\theta^s_{t+1}$
   \STATE $v^s_{t+1} = \frac{1}{B}\sum_{i=1}^B g(\tau_i\,|\,\theta^s_{t+1}) + v^s_t - \frac{1}{B}\sum_{i=1}^B g^{\omega_{\theta^s_{t+1}}}(\tau_i\,|\,\theta^s_t)$
   \ENDFOR
   \ENDFOR
   \STATE {\bfseries Output:} $\theta_{\text{out}}$ chosen uniformly at random from $\left\{ \theta^s_t \right\}_{t=1,\,s=1}^{m,\,S}$
\end{algorithmic}
\end{algorithm}

\begin{algorithm}[H]
   \caption{STORM-PG}
   \label{alg:storm-pg}
\begin{algorithmic}
   \STATE {\bfseries Input:} initial parameter $\theta_0$, large batch-size $N$, small batch-size $B$, step-size $\eta>0$, momentum parameter $\alpha\in (0,1)$
   \STATE collect $N$ trajectories with policy $\theta_0$
   \STATE $v_0 = \frac{1}{N}\sum_{i=1}^N g(\tau_i\,|\,\theta_0)$
   \FOR{$t=0$ {\bfseries to} $T-1$}
   \STATE $\theta_{t+1} = \theta_{t} + \eta v_t$
   \STATE collect $B$ trajectories with policy $\theta_{t+1}$
   \STATE $v_{t+1} = \frac{1}{B} g (\tau_i\,|\,\theta_{t+1}) + (1-\alpha)\left[ v_{t} - \frac{1}{B}\sum_{i=1}^B g^{\omega_{\theta_{t+1}}} (\tau_i\,|\,\theta_t) \right]$
   \ENDFOR
   \STATE {\bfseries Output:} $\theta_{\text{out}}$ chosen uniformly at random from $\left\{ \theta_t \right\}_{t=1}^{T}$
\end{algorithmic}
\end{algorithm}

\begin{algorithm}[H]
   \caption{PAGE-PG}
   \label{alg:page-pg}
\begin{algorithmic}
   \STATE {\bfseries Input:} initial parameter $\theta_0$, large batch-size $N$, small batch-size $B$, step-size $\eta>0$, probability $p\in (0,1]$
   \STATE collect $N$ trajectories with policy $\theta_0$
   \STATE $v_0 = \frac{1}{N}\sum_{i=1}^N g(\tau_i\,|\,\theta_0)$
   \FOR{$t=0$ {\bfseries to} $T-1$}
   \STATE $\theta_{t+1} = \theta_{t} + \eta v_t$
   \STATE $v_{t+1} = 
   \begin{cases}
    \frac{1}{N}\sum\limits_{i=1}^N g(\tau_i\,|\,\theta_{t}) & \text{prob.}\,p\\
    \frac{1}{B}\sum\limits_{i=1}^B g(\tau_i\,|\,\theta_t) + v_{t-1} - \frac{1}{B}\sum\limits_{i=1}^B g^{\omega_{\theta_t}}(\tau_i\,|\,\theta_{t-1}) & \text{prob.}\,1-p
    \end{cases}
   $
   \ENDFOR
   \STATE {\bfseries Output:} $\theta_{\text{out}}$ chosen uniformly at random from $\left\{ \theta_t \right\}_{t=1}^{T}$
\end{algorithmic}
\end{algorithm}
\section{Technical Lemmas}\label{technical lemmas}
\begin{lemma}\label{lemma 0}
Let $\theta_t$ and $\theta_{t+1}$ denote two consecutive iterates of PAGE-PG and let $g(\tau\,|\,\theta_{t+1})$ and $g^{\omega_{\theta_{t+1}}}(\tau\,|\,\theta_t)$ denote the on-policy and off-policy GPOMDP estimates computed at the iterates $\theta_{t+1}$ and $\theta_{t}$ respectively, and where $\tau\sim p(\cdot\,|\,\theta_{t+1})$. Then,
\begin{equation*}
    \mathbb{E}\left[\Big\Vert g(\tau\,|\,\theta_{t+1}) - g^{\omega_{\theta_{t+1}}}(\tau\,|\,\theta_t) \Big\Vert^2 \right] \leq C\cdot\mathbb{E}\left[\Big\Vert \theta_{t+1}-\theta_t \Big\Vert^2 \right] \,,
\end{equation*}
where $C \vcentcolon = 2(L^2 + C_{\omega}) $, $L \vcentcolon= MR/(1-\gamma)^2 + 2G^2R/(1-\gamma)^3$ and $C_{\omega} \vcentcolon = 24RG^2(2G^2+M)(W+1)\gamma/(1-\gamma)^5$.
\end{lemma}
\begin{proof}
\begin{equation*}
    \begin{aligned}
     \mathbb{E}\left[\Vert g(\tau\,|\,\theta_{t+1}) - g^{\omega_{\theta_{t+1}}}(\tau\,|\,\theta_{t})\Vert^2\right] &= \mathbb{E}\left[\Vert g(\tau\,|\,\theta_{t+1})  - g(\tau\,|\,\theta_{t}) + g(\tau\,|\,\theta_{t})-  g^{\omega_{\theta_{t+1}}}(\tau\,|\,\theta_{t})\Vert^2\right]\\
     &\overset{(a)}{\leq}  2\mathbb{E}\left[\Vert g(\tau\,|\,\theta_{t+1}) - g(\tau\,|\,\theta_{t})  \Vert^2\right] + 2\mathbb{E}\left[\Vert  g(\tau\,|\,\theta_{t}) - g^{\omega_{\theta_{t+1}}}(\tau\,|\,\theta_{t})\Vert^2\right]\\
     &\overset{(b)}{\leq}
      2 L^2 \mathbb{E}\left[\Vert \theta_{t+1}  - \theta_{t}\Vert^2\right] + 2\mathbb{E}\left[\Vert g(\tau\,|\,\theta_{t}) -  g^{\omega_{\theta_{t+1}}}(\tau\,|\,\theta_{t})\Vert^2 \right] \\
     &\overset{(c)}{\leq}  2 L^2 \mathbb{E}\left[\Vert \theta_{t+1}  - \theta_{t}\Vert^2\right] + 2C_{\omega}\mathbb{E}\left[\Vert \theta_{t+1} - \theta_{t}\Vert^2\right]\\
     &= 2(L^2 + C_{\omega} )\mathbb{E}\left[\Vert \theta_{t+1} - \theta_{t} \Vert^2\right]\,,
    \end{aligned}
\end{equation*}
where Inequality $(a)$ follows from the fact that, given any arbitrary triplet of vectors $(x, y, z)$, then $\Vert x - z + z - y \Vert^2 = \Vert x - z \Vert^2 + \Vert z - y \Vert^2 + 2\langle x-z, z-y \rangle = \Vert x - z \Vert^2 + \Vert z - y \Vert^2 + \Vert x-z \Vert^2+ \Vert z-y \Vert^2 -  \Vert x-2z+y \Vert^2 \leq 2\Vert x-z \Vert^2 + 2\Vert z-y \Vert^2$. Inequality $(b)$ is derived by considering the first point of Proposition~\ref{proposition} (see Proposition 4.2 in~\cite{xu2021} for a detailed proof). Inequality $(c)$ is obtained by considering Inequality (B.9) in~\cite{xu2021}.
\end{proof}

\begin{lemma}\label{lemma 1}
Let $v_t$ and $\theta_{t+1}$ denote the gradient estimate and the iterate generated by PAGE-PG at iteration $t+1$, respectively. Under Assumptions~\ref{assumption 1}-\ref{ass:finite_IS_var}, the estimation error at iteration $t+1$ can be bounded as follows
\begin{equation*}
    \mathbb{E}\left[ \Vert v_{t+1} - \nabla_{\theta}V(\theta_{t+1}) \Vert^2 \right] \leq (1-p) \mathbb{E}\left[ \Big\Vert v_t  - \nabla_{\theta}V(\theta_{t}) \Big\Vert^2 \right] + \frac{\eta^2(1-p)C}{B}\mathbb{E}\left[\big\Vert v_{t} \big\Vert^2  \right] + \frac{p \sigma^2}{N}\,,
\end{equation*}
where the expectation is taken with respect to all the sources of randomness up to iteration $t+1$.
\end{lemma}

\begin{proof}
Let $\mathcal{F}_t$ denote the information up to iteration $t$.  From the law of iterated expectations, we know that $\mathbb{E}\left[\Vert v_{t+1} - \nabla_{\theta}V(\theta_{t+1}) \Vert^2 \right] = \mathbb{E}\left[ \mathbb{E}\left[ \Vert v_{t+1} - \nabla_{\theta}V(\theta_{t+1}) \Vert^2 \,\big|\,\mathcal{F}_t\right]\right] $. We start by analysing the inner expectation 
\begin{align*}
     &\mathbb{E}\left[ \Big\Vert v_{t+1} - \nabla_{\theta}V(\theta_{t+1}) \Big\Vert^2 \,\Big|\,\mathcal{F}_t\right]= \\
     &=(1-p) \mathbb{E}\left[ \Big\Vert \frac{1}{B}\sum_{i=1}^B g(\tau_i\,|\,\theta_{t+1}) + v_t - \frac{1}{B}\sum_{i=1}^B g^{\omega_{\theta_{t+1}}}(\tau_i\,|\,\theta_{t}) - \nabla_{\theta}V(\theta_{t+1}) \Big\Vert^2 \,\Big|\,\mathcal{F}_t\right]\\
     & \,\, + \,\, p \,\mathbb{E}\left[ \Big\Vert  \frac{1}{N}\sum_{i=1}^N g(\tau_i\,|\,\theta_{t+1}) - \nabla_{\theta}V(\theta_{t+1}) \Big\Vert^2 \,\Big|\,\mathcal{F}_t\right]\\
     & = (1-p) \mathbb{E}\left[ \Big\Vert \frac{1}{B}\sum_{i=1}^B g(\tau_i\,|\,\theta_{t+1}) + v_t - \frac{1}{B}\sum_{i=1}^B g^{\omega_{\theta_{t+1}}}(\tau_i\,|\,\theta_{t}) - \nabla_{\theta}V(\theta_{t})+ \nabla_{\theta}V(\theta_{t})  - \nabla_{\theta}V(\theta_{t+1}) \Big\Vert^2 \,\Big|\,\mathcal{F}_t\right]\\
     &+  p \,\mathbb{E}\left[ \Big\Vert  \frac{1}{N}\sum_{i=1}^N g(\tau_i\,|\,\theta_{t+1}) - \nabla_{\theta}V(\theta_{t+1}) \Big\Vert^2 \,\Big|\,\mathcal{F}_t\right]\\
     &= (1-p) \mathbb{E}\left[ \Big\Vert v_t  - \nabla_{\theta}V(\theta_{t}) \Big\Vert^2 \,\Big|\,\mathcal{F}_t\right] + 2(1-p) \mathbb{E}\left[\Big\langle v_t - \nabla_{\theta}V(\theta_t), \,\frac{1}{B}\mathbb{E}\left[\sum_{i=1}^B g(\tau_i\,|\,\theta_{t+1})\,\Big|\,\mathcal{F}_t\right] - \nabla_{\theta}V(\theta_{t+1}) \Big\rangle \,\Big|\,\mathcal{F}_t\right]\\
     & \,\,+\,\, 2(1-p) \Big\langle v_t - \nabla_{\theta}V(\theta_t), \,-\frac{1}{B}\mathbb{E}\left[\sum_{i=1}^B g^{\omega_{\theta_{t+1}}}(\tau_i\,|\,\theta_{t})\,\Big|\,\mathcal{F}_t  \right] + \nabla_{\theta}V(\theta_{t}) \Big\rangle\\
     & \,\, + \,\,(1-p)\mathbb{E}\left[\Big\Vert \frac{1}{B}\sum_{i=1}^B g(\tau_i\,|\,\theta_{t+1})  - \frac{1}{B}\sum_{i=1}^B g^{\omega_{\theta_{t+1}}}(\tau_i\,|\,\theta_{t}) + \nabla_{\theta}V(\theta_{t})  - \nabla_{\theta}V(\theta_{t+1}) \Big\Vert^2 \,\Big|\,\mathcal{F}_t \right] \\
     & \,\, + \,\, p \,\mathbb{E}\left[ \Big\Vert  \frac{1}{N}\sum_{i=1}^N g(\tau_i\,|\,\theta_{t+1}) - \nabla_{\theta}V(\theta_{t+1}) \Big\Vert^2 \,\Big|\,\mathcal{F}_t\right]\\
     &\overset{(a)}{=} (1-p) \mathbb{E}\left[ \Big\Vert v_t  - \nabla_{\theta}V(\theta_{t}) \Big\Vert^2 \,\Big|\,\mathcal{F}_t\right] + 2(1-p) \Big\langle v_t - \nabla_{\theta}V(\theta_t), \,-\frac{1}{B}\mathbb{E}\left[ \sum_{i=1}^B g^{\omega_{\theta_{t+1}}}(\tau_i\,|\,\theta_{t}) \,\Big|\,\mathcal{F}_t\right] + \nabla_{\theta}V(\theta_{t}) \Big\rangle\\
     & \,\, + \,\,(1-p)\mathbb{E}\left[\Big\Vert \frac{1}{B}\sum_{i=1}^B g(\tau_i\,|\,\theta_{t+1})  - \frac{1}{B}\sum_{i=1}^B g^{\omega_{\theta_{t+1}}}(\tau_i\,|\,\theta_{t}) + \nabla_{\theta}V(\theta_{t})  - \nabla_{\theta}V(\theta_{t+1}) \Big\Vert^2 \,\Big|\,\mathcal{F}_t \right] \\
     & \,\, + \,\, p \,\mathbb{E}\left[ \Big\Vert  \frac{1}{N}\sum_{i=1}^N g(\tau_i\,|\,\theta_{t+1}) - \nabla_{\theta}V(\theta_{t+1}) \Big\Vert^2 \,\Big|\,\mathcal{F}_t\right]\\
     &\overset{(b)}{\leq} (1-p) \mathbb{E}\left[ \Big\Vert v_t  - \nabla_{\theta}V(\theta_{t}) \Big\Vert^2 \,\Big|\,\mathcal{F}_t\right]+ 2(1-p) \Big\langle v_t - \nabla_{\theta}V(\theta_t), \,-\frac{1}{B}\mathbb{E}\left[\sum_{i=1}^B g^{\omega_{\theta_{t+1}}}(\tau_i\,|\,\theta_{t})\,\Big|\,\mathcal{F}_t \right] + \nabla_{\theta}V(\theta_{t}) \Big\rangle\\
     & \,\, + \,\,(1-p)\mathbb{E}\left[\Big\Vert \frac{1}{B}\sum_{i=1}^B  g(\tau_i\,|\,\theta_{t+1})  - \frac{1}{B}\sum_{i=1}^B g^{\omega_{\theta_{t+1}}}(\tau_i\,|\,\theta_{t}) \Big\Vert^2 \,\Big|\,\mathcal{F}_t \right] + p \,\mathbb{E}\left[ \Big\Vert  \frac{1}{N}\sum_{i=1}^N g(\tau_i\,|\,\theta_{t+1}) - \nabla_{\theta}V(\theta_{t+1}) \Big\Vert^2 \,\Big|\,\mathcal{F}_t\right]\\
     &\overset{(c)}{\leq} (1-p) \mathbb{E}\left[ \Big\Vert v_t  - \nabla_{\theta}V(\theta_{t}) \Big\Vert^2 \,\Big|\,\mathcal{F}_t\right] 
    + 2(1-p) \Big\langle v_t - \nabla_{\theta}V(\theta_t), \,-\frac{1}{B}\mathbb{E}\left[\sum_{i=1}^B g^{\omega_{\theta_{t+1}}}(\tau_i\,|\,\theta_{t})\,\Big|\,\mathcal{F}_t\right] + \nabla_{\theta}V(\theta_{t}) \Big\rangle\\
     & \,\, + \,\,\frac{(1-p)}{B^2}\mathbb{E}\left[\sum_{i=1}^B\Big\Vert   g(\tau_i\,|\,\theta_{t+1})  -  g^{\omega_{\theta_{t+1}}}(\tau_i\,|\,\theta_{t}) \Big\Vert^2 \,\Big|\,\mathcal{F}_t \right] + \frac{p}{N^2} \,\mathbb{E}\left[\sum_{i=1}^N \Big\Vert  g(\tau_i\,|\,\theta_{t+1}) - \nabla_{\theta}V(\theta_{t+1}) \Big\Vert^2 \,\Big|\,\mathcal{F}_t\right]\,,
     \end{align*}
     
 where Equality $(a)$ follows from the fact that $\mathbb{E}_{\tau\sim p(\cdot\,|\,\theta)}\left[ g(\tau\,|\,\theta) \right] = \nabla_{\theta}V(\theta)$. Inequality $(b)$ is obtained considering that, for any random vector $X$, the variance can be bounded as follows $\mathbb{E}\left[ \Vert X - \mathbb{E}\left[ X \right] \Vert^2 \right] \leq   \mathbb{E}\left[\Vert X \Vert^2\right]$ (see Lemma B.5 in~\cite{papini2018} for a detailed proof) and Inequality $(c)$ is obtained by exploiting the triangle inequality.
By combining the derived upper bound with the results from Lemma~\ref{lemma 0} and exploiting Assumption~\ref{assumption 3}, we derive the following bound
     \begin{equation*}
     \begin{aligned}
      \mathbb{E}\left[ \Big\Vert v_{t+1} - \nabla_{\theta}V(\theta_{t+1}) \Big\Vert^2 \,\Big|\,\mathcal{F}_t\right] 
     &{\leq} (1-p) \mathbb{E}\left[ \Big\Vert v_t  - \nabla_{\theta}V(\theta_{t}) \Big\Vert^2 \,\Big|\,\mathcal{F}_t\right] \\
     & \,\,+\,\, 2(1-p) \Big\langle v_t - \nabla_{\theta}V(\theta_t), \,-\frac{1}{B}\mathbb{E}\left[\sum_{i=1}^B g^{\omega_{\theta_{t+1}}}(\tau_i\,|\,\theta_{t})\right] + \nabla_{\theta}V(\theta_{t}) \Big\rangle\\
     & \,\, + \,\,\frac{(1-p)C}{B}\mathbb{E}\left[\Big\Vert \theta_{t+1}-\theta_{t} \Big\Vert^2 \,\Big|\,\mathcal{F}_t \right] + \frac{p \sigma^2}{N}\,.
    \end{aligned}
    \end{equation*}

By considering the full expectation on the derived results, we finally obtain
\begin{equation*}
    \begin{aligned}
    \mathbb{E}\left[ \Big\Vert v_{t+1} - \nabla_{\theta}V(\theta_{t+1})\Big\Vert^2\right] 
    &\leq (1-p) \mathbb{E}\left[ \Big\Vert v_t  - \nabla_{\theta}V(\theta_{t}) \Big\Vert^2 \right] \\
     & \,\,+\,\, 2(1-p)\mathbb{E}\left[ \Big\langle v_t - \nabla_{\theta}V(\theta_t), \,-\frac{1}{B}\sum_{i=1}^B g^{\omega_{\theta_{t+1}}}(\tau_i\,|\,\theta_{t}) + \nabla_{\theta}V(\theta_{t}) \Big\rangle \right]\\
     & \,\, + \,\,\frac{(1-p)C}{B}\mathbb{E}\left[\Big\Vert \theta_{t+1}-\theta_{t} \Big\Vert^2  \right] + \frac{p \sigma^2}{N}\\
     &= (1-p) \mathbb{E}\left[ \Big\Vert v_t  - \nabla_{\theta}V(\theta_{t}) \Big\Vert^2 \right] \\
     & \,\,+\,\, 2(1-p)\mathbb{E}\left[ \mathbb{E}\left[ \Big\langle v_t - \nabla_{\theta}V(\theta_t), \,-\frac{1}{B}\sum_{i=1}^B g^{\omega_{\theta_{t+1}}}(\tau_i\,|\,\theta_{t}) + \nabla_{\theta}V(\theta_{t}) \Big\rangle \,\Big |\,\mathcal{F}_{t-1}\right] \right]\\
     & \,\, + \,\,\frac{(1-p)C}{B}\mathbb{E}\left[\Big\Vert \theta_{t+1}-\theta_{t} \Big\Vert^2  \right] + \frac{p \sigma^2}{N}\\
     &=(1-p) \mathbb{E}\left[ \Big\Vert v_t  - \nabla_{\theta}V(\theta_{t}) \Big\Vert^2 \right] + \frac{(1-p)C}{B}\mathbb{E}\left[\Big\Vert \theta_{t+1}-\theta_{t} \Big\Vert^2  \right] + \frac{p \sigma^2}{N}\\
     &=(1-p) \mathbb{E}\left[ \Big\Vert v_t  - \nabla_{\theta}V(\theta_{t}) \Big\Vert^2 \right] + \frac{\eta^2(1-p)C}{B}\mathbb{E}\left[\big\Vert v_{t} \big\Vert^2  \right] + \frac{p \sigma^2}{N}
     \end{aligned}
\end{equation*}
where the second last equality is derived considering that $\mathbb{E}_{\tau\sim p(\cdot\,|\,\theta_2)} \left[ g^{\omega_{\theta_2}}(\tau\,|\,\theta_1)\right] = \nabla_{\theta} V(\theta_1)$ for any $\theta_1,\,\theta_2$, and the last equality is obtained by considering that $\theta_{t+1} = \theta_t + \eta v_t$.
\end{proof}

\begin{lemma}\label{lemma 2}
Let Assumptions~\ref{assumption 1}-\ref{ass:finite_IS_var} hold and let $v_t$ and $\theta_{t+1}$ denote the gradient estimate and the iterate generated by PAGE-PG at iteration $t+1$, respectively. The accumulated sum of the expected estimation error satisfies the following inequality
\begin{equation*}
    \sum_{t=0}^{T-1} \mathbb{E}\left[ \Vert v_t - \nabla_{\theta}V(\theta_t) \Vert^2 \right] \leq \frac{\eta^2(1-p)C}{pB}\sum_{t=0}^{T-1}\mathbb{E}\left[ \Vert v_t\Vert^2\right] + \frac{T\sigma^2}{N} + \frac{\sigma^2}{pN}\,.
\end{equation*}
\end{lemma}

\begin{proof}
Recall that $p\in (0,1]$ is the probability that regulates the probabilistic switching. Then, 
\begin{equation}\label{eq p}
    \begin{aligned}
    &p\sum_{t=0}^{T-1} \mathbb{E}\left[ \Big\Vert  v_t - \nabla_{\theta}V(\theta_t)  \Big\Vert^2 \right] = \sum_{t=0}^{T-1} \mathbb{E}\left[ \Big\Vert  v_t - \nabla_{\theta}V(\theta_t)  \Big\Vert^2 \right] - (1-p)\sum_{t=0}^{T-1} \mathbb{E}\left[ \Big\Vert  v_t - \nabla_{\theta}V(\theta_t)  \Big\Vert^2 \right]\\
    &=\sum_{t=1}^{T} \mathbb{E}\left[ \Big\Vert  v_t - \nabla_{\theta}V(\theta_t)  \Big\Vert^2 \right] - (1-p)\sum_{t=0}^{T-1} \mathbb{E}\left[ \Big\Vert  v_t - \nabla_{\theta}V(\theta_t)  \Big\Vert^2 \right] \\
    &\quad\quad+ \mathbb{E}\left[ \Big\Vert v_0 - \nabla_{\theta}V(\theta_0) \Big\Vert^2  \right] - \mathbb{E}\left[\Big\Vert  v_T - \nabla_{\theta}V(\theta_T)  \Big\Vert^2 \right]\\
    &\leq \sum_{t=1}^{T} \mathbb{E}\left[ \Big\Vert  v_t - \nabla_{\theta}V(\theta_t)  \Big\Vert^2 \right] - (1-p)\sum_{t=0}^{T-1} \mathbb{E}\left[ \Big\Vert  v_t - \nabla_{\theta}V(\theta_t)  \Big\Vert^2 \right] + \mathbb{E}\left[ \Big\Vert v_0 - \nabla_{\theta}V(\theta_0) \Big\Vert^2  \right]\,. 
    \end{aligned}
\end{equation}
Summing up over $T$ iterations the result from Lemma~\ref{lemma 1}, we obtain the following inequality
\begin{equation}\label{eq: A2 bound}
    \sum_{t=1}^T \mathbb{E}\left[ \Big\Vert v_t - \nabla_{\theta}V(\theta_t) \Big\Vert^2 \right] \leq (1-p)\sum_{t=0}^{T-1} \mathbb{E}\left[ \Big\Vert v_t - \nabla_{\theta}V(\theta_t) \Big\Vert^2 \right] + \frac{\eta^2 (1-p)C}{B}\sum_{t=0}^{T-1}\mathbb{E}\left[ \big\Vert v_t \big\Vert^2 \right] + \frac{Tp\sigma^2}{N}\,.
\end{equation}
By combining Inequality~\eqref{eq: A2 bound} with Inequality~\eqref{eq p}, we obtain the final result
\begin{equation}
    \begin{aligned}
     p\sum_{t=0}^{T-1} \mathbb{E}\left[ \Big\Vert  v_t - \nabla_{\theta}V(\theta_t)  \Big\Vert^2 \right]
     &\leq (1-p)\sum_{t=0}^{T-1} \mathbb{E}\left[ \Big\Vert v_t - \nabla_{\theta}V(\theta_t) \Big\Vert^2 \right] + \frac{\eta^2 (1-p)C}{B}\sum_{t=0}^{T-1}\mathbb{E}\left[ \big\Vert v_t \big\Vert^2 \right] + \frac{Tp\sigma^2}{N} \\
     &\,\, - (1-p)\sum_{t=0}^{T-1} \mathbb{E}\left[ \Big\Vert  v_t - \nabla_{\theta}V(\theta_t)  \Big\Vert^2 \right] + \mathbb{E}\left[ \Big\Vert v_0 - \nabla_{\theta}V(\theta_0) \Big\Vert^2  \right]\\
     &= \frac{\eta^2 (1-p)C}{B}\sum_{t=0}^{T-1}\mathbb{E}\left[ \big\Vert v_t \big\Vert^2 \right] + \frac{Tp\sigma^2}{N}+ \mathbb{E}\left[ \Big\Vert v_0 - \nabla_{\theta}V(\theta_0) \Big\Vert^2  \right]\\
     &\leq \frac{\eta^2 (1-p)C}{B}\sum_{t=0}^{T-1}\mathbb{E}\left[ \big\Vert v_t \big\Vert^2 \right] + \frac{Tp\sigma^2}{N}+ \frac{1}{N^2}\mathbb{E}\left[\sum_{i=1}^N \Big\Vert g(\tau_i\,|\,\theta_0) - \nabla_{\theta}V(\theta_0) \Big\Vert^2  \right]\\
     &\leq \frac{\eta^2 (1-p)C}{B}\sum_{t=0}^{T-1}\mathbb{E}\left[ \big\Vert v_t \big\Vert^2 \right] + \frac{Tp\sigma^2}{N}+ \frac{\sigma^2}{N}\,,
    \end{aligned}
\end{equation}
where the second last inequality is derived by deploying the definition of $v_0$ and the triangle inequality, while the last inequality is obtained by considering Assumption~\ref{assumption 3}.

\end{proof}

\section{Proof of the Main Theoretical Results}\label{proof main theory}

\begin{proof}[Proof of Theorem~\ref{theorem 1}]
In the considered setting we know from Proposition~\ref{proposition} that $V(\theta)$ is $L$-smooth, where $L \vcentcolon = MR/(1-\gamma)^2 + 2G^2R/(1-\gamma)^3$. Consequently, we can write the following lower bound on $V(\theta_{t+1})$

\begin{equation*}
    \begin{aligned}
     V(\theta_{t+1})&\geq V(\theta_t) + \big\langle \nabla V(\theta_t) , \theta_{t+1}-\theta_t  \big\rangle - \frac{L}{2}\Vert \theta_{t+1}-\theta_t\Vert^2\\
     &=V(\theta_t) + \eta\big\langle \nabla V(\theta_t) , v_t  \big\rangle - \frac{\eta^2 L}{2}\big\Vert v_t \big\Vert^2\\
     &\overset{(a)}{=}V(\theta_t) + \frac{\eta}{2}\big\Vert \nabla V(\theta_t) \big\Vert^2 + \frac{\eta}{2}\big\Vert v_t \big\Vert^2 -  \frac{\eta}{2}\big\Vert v_t -  V(\theta_t)\big\Vert^2  - \frac{\eta^2 L}{2}\big\Vert v_t \big\Vert^2\\
     &= V(\theta_t) + \frac{\eta}{2}\big\Vert \nabla V(\theta_t) \big\Vert^2 + \frac{\eta}{2}\left(1-\eta L  \right)\big\Vert v_t \big\Vert^2 -  \frac{\eta}{2}\big\Vert  v_t - V(\theta_t) \big\Vert^2\,,
    \end{aligned}
\end{equation*}
where Equality (a) is derived by considering that $\langle x, y \rangle = \frac{\Vert x \Vert^2}{2} + \frac{\Vert y \Vert^2}{2} - \frac{\Vert y - x \Vert^2}{2}$.
Since by design choice $\eta \leq 1/(2L)$, then
\begin{equation*}
    \begin{aligned}
     V(\theta_{t+1})
     &\geq V(\theta_t) + \frac{\eta}{2}\big\Vert \nabla V(\theta_t) \big\Vert^2 + \frac{\eta}{4}\big\Vert v_t \big\Vert^2 -  \frac{\eta}{2}\big\Vert v_t - V(\theta_t) \big\Vert^2\,.
    \end{aligned}
\end{equation*}
By rearranging some terms and changing the direction of the inequality, we obtain
\begin{equation}\label{eq: upbound V}
    V(\theta_t) - V(\theta_{t+1})\leq -\frac{\eta}{2}\big\Vert \nabla V(\theta_t) \big\Vert^2 - \frac{\eta}{4}\big\Vert v_t \big\Vert^2 +  \frac{\eta}{2}\big\Vert v_t - V(\theta_t) \big\Vert^2\,.
\end{equation}
Summing up over the first $T$ iterations, we obtain
\begin{equation}\label{eq: sum over T}
    \sum_{t=0}^{T-1}\left( V(\theta_t) - V(\theta_{t+1})\right) = V(\theta_0) - V(\theta_T) \leq -\frac{\eta}{2}\sum_{t=0}^{T-1}\big\Vert \nabla V(\theta_t) \big\Vert^2 - \frac{\eta}{4}\sum_{t=0}^{T-1} \big\Vert v_t \big\Vert^2 +  \frac{\eta}{2}\sum_{t=0}^{T-1}\big\Vert v_t - V(\theta_t) \big\Vert^2\,.
\end{equation}
By taking the expectation on both sides of Equation~\eqref{eq: sum over T} and considering the fact that $V^*\geq V(\theta)$ for all $\theta\in\mathbb{R}^d$, we get
\begin{equation}
\begin{aligned}
    V(\theta_0) - V^* &\leq  V(\theta_0) - \mathbb{E}\left[ V(\theta_T)\right] \\
    &\leq -\frac{\eta}{2}\sum_{t=0}^{T-1}\mathbb{E}\left[\big\Vert \nabla V(\theta_t) \big\Vert^2\right] - \frac{\eta}{4}\sum_{t=0}^{T-1}\mathbb{E}\left[ \big\Vert v_t \big\Vert^2\right] +  \frac{\eta}{2}\sum_{t=0}^{T-1}\mathbb{E}\left[\big\Vert v_t - V(\theta_t) \big\Vert^2\right]\\
    &\overset{(a)}{\leq} -\frac{\eta}{2}\sum_{t=0}^{T-1}\mathbb{E}\left[\big\Vert \nabla V(\theta_t) \big\Vert^2\right] - \frac{\eta}{4}\sum_{t=0}^{T-1}\mathbb{E}\left[ \big\Vert v_t \big\Vert^2\right] + \frac{\eta^3 (1-p)C}{2pB}\sum_{t=0}^{T-1} \mathbb{E}\left[ \big\Vert v_t \big\Vert^2 \right] + \frac{\eta T\sigma^2}{2N} + \frac{\eta\sigma^2}{2pN}\\
    &= -\frac{\eta}{2}\sum_{t=0}^{T-1}\mathbb{E}\left[\big\Vert \nabla V(\theta_t) \big\Vert^2\right]  -\frac{\eta}{2}\left(\frac{1}{2}- \frac{\eta^2 (1-p)C}{pB}\right)\sum_{t=0}^{T-1} \mathbb{E}\left[ \big\Vert v_t \big\Vert^2 \right] + \frac{\eta T\sigma^2}{2N} + \frac{\eta\sigma^2}{2pN}\\
    &\overset{(b)}{\leq} -\frac{\eta}{2}\sum_{t=0}^{T-1}\mathbb{E}\left[\big\Vert \nabla V(\theta_t) \big\Vert^2\right]  + \frac{\eta T\sigma^2}{2N} + \frac{\eta\sigma^2}{2pN}\,, 
\end{aligned}
\end{equation}
where Inequality $(a)$ follows from Lemma~\ref{lemma 2} and Inequality $(b)$ follows from the fact that $\frac{\eta^2 (1-p)}{pB}\leq 1/(2C)$.
Finally, rearranging the terms and multiplying both sides by $\frac{2}{\eta T}$, we obtain the final result
\begin{equation*}
\begin{aligned}
    \frac{1}{T}\sum_{t=0}^{T-1}\mathbb{E}\left[\big\Vert \nabla V(\theta_t) \big\Vert^2\right]  &\leq  \frac{2\left(V^* - V(\theta_0)\right)}{\eta T} + \frac{\sigma^2}{N} + \frac{\sigma^2}{pN T}\,.
\end{aligned}
\end{equation*}
\end{proof}

\begin{proof}[Proof of Corollary~\ref{corollary 1}]
Let $\Delta = V^* - V(\theta_0)$. We set $p=\frac{1}{N}$ and $\eta = \frac{\sqrt{B}}{\sqrt{2CN}}$. Notice that with this choice of parameters, we verify the constraints on the parameter selection, i.e. $\eta^2\leq \min\left\{ \frac{Bp}{2C(1-p)}\,,\frac{1}{4L^2} \right\}$. In particular, with this choice for the probability of switching and the step-size, we get that $\eta^2 = \frac{B}{2CN} = \frac{Bp}{2C}\leq \frac{Bp}{2C(1-p)}$. In addition, since $\frac{B}{N}\leq 1$ and $C = 2(C_{\omega} + L^2)$, then $\eta^2\leq \frac{1}{4L^2}$.

We now want to derive some values of $T$ and $N$ such that
\begin{equation}\label{eq: impose the bound}
    \frac{1}{T}\sum_{t=0}^{T-1}\mathbb{E}\left[\Big\Vert \nabla_{\theta}V(\theta_T) \Big\Vert^2\right]\leq \epsilon^2\,.
\end{equation}
For that, we utilize the results of Theorem~\ref{theorem 1} as follows
\begin{equation}\label{eq: impose epsilon ub}
    \frac{2\Delta}{\eta T} + \frac{\sigma^2}{N} + \frac{\sigma^2}{p N T} \leq \epsilon^2\,.
\end{equation}
We decide to partition $\epsilon^2$ equally among the three terms in~\eqref{eq: impose epsilon ub} 
\begin{equation}\label{eq: cor 2 eq 1}
\begin{aligned}
    \frac{2\Delta}{\eta T}\leq \frac{\epsilon^2}{3}\,,\quad\frac{\sigma^2}{N}\leq \frac{\epsilon^2}{3}\,,\quad\frac{\sigma^2}{p N T}\leq \frac{\epsilon^2}{3}\,,
\end{aligned}
\end{equation}
from which we infer that, with these specific choices, Inequality~\eqref{eq: impose the bound} is verified for all values of $N$ and $T$ such that
\begin{align}\label{eq: lower bound T}
    N &\geq \frac{3\sigma^2}{\epsilon^2}\,,&
    T &\geq \max\left\{ \frac{6\Delta}{\epsilon^2}\frac{\sqrt{2CN}}{\sqrt{B}}\,,\frac{3\sigma^2}{\epsilon^2} \right\}\,.
\end{align}
For the sake of compactness, we define $K_1 = 6\Delta \sqrt{2C}$ and $K_2  = 3\sigma^2$. The average number of trajectories over $T$ iterations is given by the following expression
\begin{equation}\label{eq: number average trajectories}
    p T N + (1-p) T B = T\left( p N + (1-p)B \right)\,.
\end{equation}
We want to study the average sample complexity of PAGE-PG to reach an $\epsilon$-stationary solution when $\epsilon\rightarrow 0$.
To do that, we set $T =   \frac{K_1}{\epsilon^2}\frac{\sqrt{N}}{\sqrt{B}}+\frac{K_2}{\epsilon^2}$, such that the constraints on $T$ from Equation~\eqref{eq: lower bound T} are verified. Finally, by plugging this value for the number of iterations and the choice of $p$ in Equation~\eqref{eq: number average trajectories}, we get
\begin{equation}
\begin{aligned}
    T\left(pN + (1-p)B \right) &=  T\left( 1 + B - \frac{B}{N}\right)\\
    &= \left(\frac{K_1}{\epsilon^2}\frac{\sqrt{N}}{\sqrt{B}}+\frac{K_2}{\epsilon^2} \right)\left( 1 + B - \frac{B}{N}\right)\,.
\end{aligned}
\end{equation}
By setting the batch-size parameters to $B= \mathcal{O}\left( 1\right)$ and $N= \mathcal{O}\left( \epsilon^{-2}\right)$, we finally get
\begin{equation}
\begin{aligned}
    \left(\frac{K_1}{\epsilon^2}\frac{\sqrt{N}}{\sqrt{B}}+\frac{K_2}{\epsilon^2} \right)\left( 1 + B - \frac{B}{N}\right) &=\mathcal{O}\left(\epsilon^{-3} \right)\mathcal{O}\left( 1\right) \\
    &= \mathcal{O}\left( \epsilon^{-3}\right)\,.
\end{aligned}
\end{equation}

\end{proof}

\begin{proof}[Proof of Corollary~\ref{corollary 2}]
We combine the gradient dominancy condition and the results from Theorem~\ref{theorem 1} as follows
\begin{equation}
\begin{aligned}
    V^* - \max_{t\leq T}\mathbb{E}\left[ V(\theta_t)\right] & \leq V^* - \mathbb{E}\left[ V(\theta_{a})  \right]\\
    &\leq \lambda \mathbb{E}\left[ \big\Vert \nabla_{\theta}V(\theta_a) \big\Vert^2 \right]\\
    &\leq \frac{\lambda}{T}\sum_{t=0}^{T-1}\mathbb{E}\left[\big\Vert \nabla_{\theta}V(\theta_t) \big\Vert^2 \right]\\
    &\leq \frac{2\lambda \left(V^* - V(\theta_0) \right)}{\eta T} + \frac{\sigma^2 \lambda}{N} + \frac{\sigma^2 \lambda}{p N T}\,,
\end{aligned}
\end{equation}
where $a:=\arg\min_{t\leq T} \mathbb{E}\left[\big\Vert \nabla V(\theta_t)   \big\Vert^2 \right]$.

\end{proof}

\section{Additional Details on the Hyperparameters}\label{details on benchmarks}
\begin{table}[H]
\caption{Hyperparameter setting for the Acrobot benchmark.}
\vskip 0.15in
\begin{center}
\begin{small}
\begin{sc}
\begin{tabular}{lccccr}
\toprule
\textbf{Method} & \textbf{$\eta$} & \textbf{$\alpha$} & \textbf{$p_t$} \\
\midrule
GPOMDP    & $10^{-4}$ & -- & --\\
SVRPG    & $10^{-5}$ & -- & --\\
SRVRPG & $6\cdot 10^{-6}$ & -- & --\\
STORM-PG    & $10^{-4}$ & 0.9 & -- \\
PAGE-PG   & $5\cdot 10^{-6}$ & -- & $0.01,\,0.4$ \\
\bottomrule
\end{tabular}
\end{sc}
\end{small}
\end{center}
\vskip -0.1in
\end{table}

\begin{table}[H]
\caption{Hyperparameter setting for the Cartpole benchmark.}
\vskip 0.15in
\begin{center}
\begin{small}
\begin{sc}
\begin{tabular}{lccccr}
\toprule
\textbf{Method} & \textbf{$\eta$} & \textbf{$\alpha$} & \textbf{$p_t$} \\
\midrule
GPOMDP    & $10^{-4}$ & -- & --\\
SVRPG    & $10^{-4}$ & -- & --\\
SRVRPG & $10^{-6}$ & -- & --\\
STORM-PG    & $4\cdot 10^{-5}$ & 0.99 & -- \\
PAGE-PG   & $5\cdot 10^{-5}$ & -- & $0.8$ \\
\bottomrule
\end{tabular}
\end{sc}
\end{small}
\end{center}
\vskip -0.1in
\end{table}

\end{document}